\documentclass[letterpaper, 10pt, conference]{ieeeconf}
\IEEEoverridecommandlockouts
\usepackage{cite}

\usepackage{amssymb,amsfonts,amsmath,amsthm}
\usepackage{mathtools}
\usepackage{enumerate}
\usepackage{verbatim}
\usepackage{commath}
\usepackage[mathcal]{euscript}
\usepackage[usenames]{color}

\definecolor{plum} {rgb}{.4,0,.4}
\definecolor{BrickRed} {rgb}{0.6,0,0}

\usepackage[plainpages=false,pdfpagelabels,colorlinks=true,linkcolor=BrickRed,citecolor=plum]{hyperref}

\def\ddefloop#1{\ifx\ddefloop#1\else\ddef{#1}\expandafter\ddefloop\fi}

\def\ddef#1{\expandafter\def\csname b#1\endcsname{\ensuremath{\boldsymbol{#1}}}}
\ddefloop ABCDEFGHIJKLMNOPQRSTUVWXYZabcdefghijklmnopqrstuvwxyz\ddefloop

\def\ddef#1{\expandafter\def\csname c#1\endcsname{\ensuremath{\mathcal{#1}}}}
\ddefloop ABCDEFGHIJKLMNOPQRSTUVWXYZ\ddefloop
 
\def\ddef#1{\expandafter\def\csname s#1\endcsname{\ensuremath{\mathsf{#1}}}}
\ddefloop ABCDEFGHIJKLMNOPQRSTUVWXYZ\ddefloop

\def\Naturals{{\mathbb N}}
\def\Reals{{\mathbb R}}
\def\p{{\partial}} 
\def\Ex{{\mathbf E}} 
\def\Pr{{\mathbf P}} 
\def\eps{\varepsilon}

\DeclareMathOperator*{\argmin}{arg\,min}

\newtheorem{theorem}{Theorem}
\newtheorem{assumption}{Assumption}
\newtheorem{example}{Example}
\newtheorem{proposition}{Proposition}
\newtheorem{lemma}{Lemma}

\newtheorem{remark}{Remark}

\title{\LARGE \textbf{Fitting an immersed submanifold to data via Sussmann's orbit theorem}}

\author{Joshua Hanson and Maxim Raginsky%
\thanks{This work was supported in part by the NSF under award CCF-2106358 (``Analysis and Geometry of Neural Dynamical Systems'') and in part by the Illinois Institute for Data Science and Dynamical Systems
(iDS${}^2$), an NSF HDR TRIPODS institute, under award CCF-1934986.}%
\thanks{Joshua Hanson is with the Department of Electrical and Computer Engineering and the Coordinated Science Laboratory,
        University of Illinois at Urbana-Champaign, 1308 W Main St, Urbana, IL 61801
        {\tt\small jmh4@illinois.edu}}%
\thanks{Maxim Raginsky is with the Department of Electrical and Computer Engineering and the Coordinated Science Laboratory,
        University of Illinois at Urbana-Champaign, 1308 W Main St, Urbana, IL 61801
        {\tt\small maxim@illinois.edu}}%
}

\begin{document}

\maketitle


\begin{abstract}
        
This paper describes an approach for fitting an immersed submanifold of a finite-dimensional Euclidean space to random samples. The reconstruction mapping from the ambient space to the desired submanifold is implemented as a composition of an {\em encoder} that maps each point to a tuple of (positive or negative) times and a {\em decoder} given by a composition of flows along finitely many vector fields starting from a fixed initial point. The encoder supplies the times for the flows. The encoder-decoder map is obtained by empirical risk minimization, and a high-probability bound is given on the excess risk relative to the minimum expected reconstruction error over a given class of encoder-decoder maps. The proposed approach makes fundamental use of Sussmann's orbit theorem, which guarantees that the image of the reconstruction map is indeed contained in an immersed submanifold. 
   
\end{abstract}

\section{Introduction}

The \textit{manifold learning problem} can be stated as follows: A point cloud in $\Reals^d$ is given, and we wish to construct a smooth $m$-dimensional submanifold of $\Reals^d$ (where $m < d$) to approximate this point cloud. This problem has received a great deal of attention in the machine learning community, where such representations can serve as intermediate objects in multistage inference procedures, their lower dimensionality conferring computational advantages when the ambient dimension $d$ is high. Most existing approaches \cite{RowSau00,TenSilLan00,BelNiy03,DonGri03,Niyogi_2008,Fefferman2020} attempt to construct a local description of the approximating manifold via an atlas of charts, and an additional step is needed to piece the charts together into a global description. 

There have been several recent proposals for manifold learning relying on deep neural nets \cite{SchmidtHieber2019,Buchanan2021,Dikkala_2021,chen_2022}, which are currently the dominant modeling paradigm in machine learning. However, these approaches are also local in nature. In this work, we propose an alternative \textit{global} procedure for fitting low-dimensional submanifolds to data via a deep and fundamental result in differential geometry --- namely, Sussmann's orbit theorem \cite{sussmann_1973}. While we give the precise statement of the orbit theorem in the next section, the underlying idea is as follows: Given an arbitrary collection $\cF$ of smooth vector fields on a smooth finite-dimensional manifold $M$, the \textit{orbit of $\cF$} through a point $\xi$ of $M$, i.e., the set of all points attainable via successive forward and backward finite-time motions along a finite number of vector fields in $\cF$ starting from $\xi$, has a natural structure of an \textit{immersed submanifold} of $M$. The Chow-Rashevskii theorem \cite{chow_1939}, \cite{rashevskii_1938}, which states that, if the Lie algebra of a collection of vector fields evaluated at any point spans the tangent space to $M$ at that point, then the orbit is equal to the entire manifold, i.e., the flow is transitive, is a corollary of this result.

This suggests the following natural recipe for fitting an immersed submanifold of $\Reals^d$ to a finite point cloud $S = \{x_1,\ldots,x_n\} \subset \Reals^d$:
\begin{enumerate}
	\item Fix a family $\cF$ of smooth vector fields on $\Reals^d$ and a family $\cA$ of smooth functions from $\Reals^d$ into a finite interval $[T_0,T_1]$ containing $0$.
	\item Fix a positive integer $m < d$, which will serve as an upper bound on the dimension of the submanifold.
	\item Find $m$-tuples $(f_1,\ldots,f_m) \in \cF^m$ and $(a_1,\ldots,a_m) \in \cA^m$ and a starting point $\xi$ in the convex hull of $S$ to minimize the average reconstruction error
        \begin{align*}
            \frac{1}{n}\sum^n_{i=1} |x_i - e^{a_m(x_i)f_i} \circ \dots \circ e^{a_1(x_i)f_1}\xi|,
        \end{align*}
        where $|\cdot|$ denotes the Euclidean norm on $\Reals^d$, and $e^{tf}$ denotes the flow map of $f$, i.e., the solution at time $t$ of the ODE $\dot{x} = f(x)$ starting from $x(0) = \xi$.
\end{enumerate}
This procedure produces an explicit \textit{encoder} map from $\Reals^d$ into $[T_0,T_1]^m$ given by $x \mapsto (a_1(x),\ldots,a_m(x))$ and an explicit \textit{decoder} map from the $m$-cube $[T_0,T_1]^m$ into the orbit of $\{f_1,\ldots,f_m\}$ through $\xi$ given by $(t_1,\ldots,t_m) \mapsto e^{t_m f_m} \circ \dots \circ e^{t_1f_1}\xi$. Sussmann's theorem then guarantees that the decoder maps the cube  $[T_0,T_1]^m$ into an immersed submanifold of $\Reals^d$. Moreover, the  \textit{reconstruction} of a point $x$ as $e^{a_m(x)f_m} \circ \dots \circ e^{a_1(x)f_1}\xi$ is realized as a composition of flow maps of $m$ time-homogeneous ODEs starting from $\xi$, where the duration and the direction (forward or backward) of each flow is determined by the target point $x$.

In this work, we consider the statistical learning setting, where the points of $S$ are independent samples from an unknown probability measure supported on a compact subset of $\Reals^d$. The encoders $(a_1,\ldots,a_m)$, the vector fields $(f_1,\ldots,f_m)$, and the starting point $\xi$ are obtained by minimizing the empirical risk on $S$, and we give a high-probability bound on the excess risk relative to the best expected reconstruction error. To accomplish this, we recruit control-theoretic techniques to quantify how the reconstruction error propagates forward through a composition of flow maps of vector fields. The bounds obtained from this analysis depend on the regularity conditions satisfied by the families of encoders and vector fields.

Section \ref{sec:architecture} gives the necessary background on the orbit theorem, describes in more detail the encoder-decoder architecture, and states the learning problem. Section~\ref{sec:examples} contains the main result and some illustrative examples. Section \ref{sec:complexity} is devoted to the proof of the main result, and some concluding remarks follow in Section \ref{sec:conclusions}.

\section{Model description and problem statement}\label{sec:architecture}

To motivate the encoder-decoder architecture, we review the setting and statement of Sussmann's orbit theorem \cite{sussmann_1973}. Let $\cF$ be a family of smooth vector fields on a finite-dimensional smooth manifold $M$. We assume that, for each $f \in \cF$, the flow map $e^{tf} : M \to M$ is $\textit{complete}$, i.e., defined for all times $t \in \Reals$. Consider the group ${\mathbb G}(\cF)$ of diffeomorphisms of $M$ generated by all such flow maps, i.e.,
\begin{align*}
	{\mathbb G}(\cF) &:= \big\{ e^{t_k f_k} \circ \dots \circ e^{t_2 f_2} \circ e^{t_1 f_1} : k \in \Naturals,\nonumber\\
	&\quad \quad \quad t_1,\dots,t_k \in \Reals,\ f_1,\dots,f_k \in \cF \big\}.
\end{align*}
We define the \textit{orbit} of $\cF$ through a point $\xi \in M$ as the set
\begin{equation*}
\begin{split}
    \cO_\xi &= \{ g(\xi) : g \in {\mathbb G}(\cF) \} \subseteq M.
\end{split}
\end{equation*}
The orbit theorem tells us that $\cO_\xi$ carries a canonical topological structure:
\begin{theorem}[Sussmann \cite{sussmann_1973}] For each point $\xi \in M$, the orbit $\cO_\xi$ is a connected immersed submanifold of $M$. The tangent space to $\cO_\xi$ at the point $x$ is the linear subspace of $T_x M$ spanned by vectors $g_*f(x)$, $f \in \cF$, $g \in {\mathbb G}(\cF)$, where $g_*f : M \to TM$ is the pushforward of $f$ by $g$.
\end{theorem}
\noindent One of the main applications of the orbit theorem is in geometric control \cite{agrachev_sachkov_2004}. Let $\dot{x} = f(x,u)$, $x \in M$, $u \in U$ be a smooth controlled system on $M$, where $U$ is some control set. Then one applies the orbit theorem to  $\cF = \{f(\cdot,u) : M \to TM : u \in U\}$. Notice that $\cO_\xi$ is in general larger than the reachable set from $\xi$, because both forward- and backward-in-time motions are permitted, whereas the reachable set only allows forward motions.

For our purposes, though, Sussmann's theorem provides a natural recipe for constructing immersed submanifolds in a global fashion, as opposed to a local description based on an atlas of charts: Given an ambient manifold $M$, we choose a finite family $\tilde{\cF} = \{f_1,\ldots,f_m\}$ of vector fields on $M$, an initial condition $\xi \in M$, and a finite interval $[T_0,T_1]$ of the real line containing $0$, and then consider a map
$$
[T_0,T_1]^m \ni (t_1,\ldots,t_m) \mapsto e^{t_m f_m} \circ \dots \circ e^{t_2 f_2} \circ e^{t_1 f_1}\xi \in M.
$$
This map sends the $m$-cube $[T_0,T_1]^m$ into a subset of the orbit $\cO_\xi$ of $\tilde{\cF}$ through $\xi$, which is an immersed submanifold of $M$ by the orbit theorem. It should be emphasized, however, that the map $(t_1,\ldots,t_m) \mapsto e^{t_m f_m} \circ \dots e^{t_1f_1}\xi$ is not an immersion, unless the vectors $f_1(\xi),\dots,f_m(\xi)$ are linearly independent and the point $(t_1,\ldots,t_m)$ lies in a sufficiently small neighborhood of $0$ in $\Reals^m$.

\subsection{The encoder-decoder architecture}

Inspired by this observation, we describe an encoder-decoder architecture for unsupervised learning of immersed submanifolds. Let $\cA$ be a family of smooth maps $a : M \to [T_0,T_1]$. For example, $\cA$ could be a collection of linear maps, polynomials, neural nets, etc. (the precise definitions of these families of course will depend on a particular parameterization of $M$). Then, for any $a_1,\ldots,a_m \in \cA$, the product map
\begin{gather*}
    \ba := a_1 \times \cdots \times a_m : M \to [T_0,T_1]^m\\
    x \mapsto \ba(x) = (a_1(x),\dots,a_m(x)).
\end{gather*}
will be an \textit{encoder} that ``represents'' a point $x \in M$ as a (possibly lower-dimensional) tuple of times $\ba(x) = (a_1(x),\ldots,a_m(x)) \in [T_0,T_1]^m$.

Now let $\cF$ be a family of vector fields $f : M \to TM$, and let $\xi \in M$ be a fixed initial condition. For example, $\cF$ could be a collection of constant vector fields, linear vector fields, neural nets, a finite set of fixed vector fields, etc. (again these families are defined in terms of a parameterization of $M$). Then any composition of flows of the form
\begin{gather*}
    g = g_{\bf,\xi} : [T_0,T_1]^m \to M \\
    \bt = (t_1,\dots,t_m) \mapsto e^{\bt \bf} \xi := e^{t_m f_m} \circ \dots \circ e^{t_1 f_1} \xi,
\end{gather*}
will be a \textit{decoder}, which outputs the state obtained from starting at $\xi$ and iteratively flowing along each vector field $f_j$ for duration $t_j$, for $j=1,\dots,m$. Observe that $g(\Reals,\dots,\Reals) \subseteq \cO_\xi$, i.e., the image of the decoder $g$ is (a subset of) the orbit of $\tilde{\cF} = \{f_1,\ldots,f_m\}$ through $\xi$, which is an immersed submanifold of $M$. We will refer to the composition $G_{\ba,\bf,\xi} := g_{\bf,\xi} \circ \ba : M \to M$ as a \textit{reconstruction map}. The main idea here is that, by choosing $\ba \in \cA^m$, $\bf \in \cF^m$, and $\xi \in M$, we automatically choose an immersed submanifold of $M$, namely the orbit $\cO_\xi$ of $\{f_1,\ldots,f_m\}$ through $\xi$, and the reconstruction map $G_{\ba,\bf,\xi}$ then allows us to map any point $x \in M$ to a point $\widehat{x} = G_{\ba,\bf,\xi}(x) \in \cO_\xi$.

\subsection{The learning problem and the basic excess risk bound}

Now we proceed to formulate the learning problem. Consider a probability measure $\mu$ that is compactly supported on $M = \Reals^d$, and a collection $S$ of independent and identically distributed samples $X_1, \dots, X_n \sim \mu$. We denote the support by $K := \text{supp}(\mu) \subset \Reals^d$. We would like to fit an immersed submanifold of $M$ to the data $S$.
	
Suppose we have some families $\cA$ and $\cF$ from which the encoders $a_1,\dots,a_m$ and vector fields $f_1,\dots,f_m$ will be drawn, where $m$ and $[T_0,T_1]$ are fixed in advance; the initial condition $\xi$ will be drawn from $K$. We will denote by $\cG$ the set of all reconstruction maps $G_{\ba,\bf,\xi} : M \to M$. We can then define the \textit{expected risk} of $G \in \cG$,
\begin{equation*}
\begin{split}
    L_\mu(G) :\!\!&= \Ex_\mu [|X-G(X)|] = \int_K |x-G(x)| \mu(\dif x),
\end{split}
\end{equation*}
as well as the \textit{minimum risk}
\begin{equation*}
    L_\mu^*(\cG) := \inf_{G \in \cG} L_\mu(G).
\end{equation*}

The minimum risk measures in some sense the inherent expressiveness of the model class $\cG$. Given $\mu$, we can find some $G^* \in \cG$ that (approximately) achieves $L^*_\mu(\cG)$. However, $\mu$ is in general unknown, so we attempt to learn $G^*$ from $S$ by minimizing the \textit{empirical risk} $\frac{1}{n} \sum_{i=1}^n|X_i - G(X_i)|$ over $\cG$. Denote
\begin{equation*}
    \widehat{G} \in \argmin_{G \in \cG} \frac{1}{n} \sum_{i=1}^n |X_i - G(X_i)|.
\end{equation*}
Observe that the empirical risk is a random variable, because the data $X_1,\dots,X_n$ are random. In practice, one attempts to approximate $\widehat{G}$ using a numerical optimization routine such as gradient descent.

The generalization capacity of the model class $\cG$ --- which represents, roughly speaking, how well models tend to perform on data that are not seen during training but are drawn from the same distribution --- is captured by its \textit{empirical Rademacher complexity}
\begin{equation*}
	\cR_n(\cG) := \Ex \Bigg[\sup_{G \in \cG}  \frac{1}{n} \sum_{i=1}^n \varepsilon_i |X_i - G(X_i)| \Bigg| S \Bigg],
\end{equation*}
where $\varepsilon_1,\dots,\varepsilon_n$ are independent Rademacher random variables, i.e., $\Pr(\varepsilon_i = \pm 1)= \frac{1}{2}$, that are independent of the training data $S$.  This quantity measures how well the model class is capable of fitting random noise --- if the Rademacher complexity is large, then the learned model may generalize poorly to unseen data, whereas if the Rademacher complexity is small, then the model class is less able to exhibit high variability and perhaps generalizes better. The following excess risk bound is standard (see, e.g., \cite{raginsky_2021}):
\begin{theorem}\label{thm:basic_bound}
    Assume that $|x - G(x)|$ is bounded between $0$ and $B < \infty$ for all $x \in K, G \in \cG$. Then the following excess risk guarantee holds with probability at least $1-\delta$:
    \begin{equation*}
        L_\mu(\widehat{G}) \leq L^*_\mu(\cG) + 4\, \Ex \cR_n(\cG) + B \sqrt{\frac{2 \log(\frac{1}{\delta})}{n}}.
    \end{equation*}
\end{theorem}
  
\noindent The main objective of this work is to bound the Rademacher complexity of the class $\cG$ of encoder-decoder pairs with architecture as described previously. One can expect that the properties of and/or restrictions on the families $\cA$ and $\cF$, and the number of `layers' $m$, will affect the Rademacher complexity. Namely, if $\cA$ and $\cF$ are themselves very expressive, then $\cR_n(\cG)$ will tend to be larger, whereas if $\cA$ and $\cF$ are simple, then $\cR_n(\cG)$ will tend to be smaller, and we wish to quantify this relationship. We will suppress absolute constants (i.e., ones that do not depend on any parameters of the problem) by writing $a \lesssim b$ as a shorthand for $a \le Cb$ for some absolute constant $C > 0$.
    
\section{Main result and examples}\label{sec:examples}

We impose the following assumptions on $\cA$ and $\cF$:

\begin{assumption}\label{as:cA}
    $\cA$ is equicontinuous on $K$.
\end{assumption}
  
\begin{assumption}\label{as:cF} 
    There exists a compact set $\tilde{K} \supseteq K$, such that $e^{\bt \bf}\xi \in \tilde{K}$ for all $\bt \in [T_0,T_1]^m$, $\bf \in \cF^m$, $\xi \in K$.  Moreover, $\cF$ is uniformly bounded and equi-Lipschitz: there exist finite constants $L_0,L$, such that for every $f \in \cF$ 
    \begin{itemize}
        \item $|f(x)| \leq L_0$ \normalfont{for all} $x \in \Reals^d$
        \item $|f(x) - f(y)| \leq L |x - y|$ \normalfont{for all} $x,y \in \Reals^d$.
    \end{itemize}
\end{assumption}

\begin{remark}\label{rem:bump} {\em If $\cF$ is uniformly bounded and equi-Lipschitz on $\tilde{K}$ but not necessarily on $\Reals^d$, we can always ensure global boundedness and Lipschitz continuity by multiplying each $f \in \cF$ by a $C^\infty$ bump function $\rho$ satisfying $\rho(\tilde{K}) \equiv 1$ and $\rho(\Reals^d \setminus \bar{K}) \equiv 0$ for some compact set $\bar{K}$ containing $\tilde{K}$, so that the flows $e^{t \rho f} \xi$ are unaffected on $\tilde{K}$  provided that $\xi \in K$.} \end{remark}

\noindent Assumption \ref{as:cF} suffices to guarantee the existence and uniqueness of flow maps and the existence of a continuous \textit{comparison function}
\begin{equation*}
    \beta : [0,\infty) \times \Reals \to [0,\infty)
\end{equation*}
such that for every $f \in \cF$ and every $\xi,\xi' \in \Reals^d$, we have
\begin{itemize}
    \item $|e^{t f} \xi - e^{t f} \xi'| \leq \beta(|\xi - \xi'|, t)$
    \item $\beta(r,0) = r$ for each $r \in [0,\infty)$
    \item $\beta(0,t) = 0$ for each $t \in \Reals$
    \item $r \mapsto \beta(r,t)$ is monotonically increasing for each $t \in \Reals$
\end{itemize}
The second and third conditions actually follow from the first by setting $t = 0$ or $\xi = \xi'$, respectively, and the last one follows from properties of ODEs given that the bound must hold for all initial conditions and times. We will also assume that $\beta$ is right-differentiable at zero in the first argument, which is without loss of generality due to the following paragraph.

Assuming nothing more about $\cF$, we can form a worst-case estimate of $\beta$. A standard argument using Gr\"onwall's inequality \cite{hirsch_smale_1995} shows that if $f$ is globally $L$-Lipschitz and $x,x'$ are solutions to the following initial value problems
\begin{alignat*}{2}
    \dot{x} &= f(x),\quad & x(s) &= \xi \\
    \dot{x}' &= f(x'),\quad & x'(s) &= \xi',
\end{alignat*}
then we have
\begin{equation*}
    |x(t) - x'(t)| \leq |\xi - \xi'| e^{L |t - s|}.
\end{equation*}
Therefore $\beta(r,t) = r e^{L |t|}$ is a suitable comparison function. In fact, since every vector field in $\cF$ is bounded in magnitude by $L_0$, $\beta(r,t) = \min\{r e^{L |t|}, r + 2 L_0 |t|\}$ also works. However, it may be possible to improve upon this estimate depending on other properties of $\cF$. This is explored further in the examples at the end of this section.
	
We will equip the sets $\cA\big|_K$ and $\cF\big|_{\tilde{K}}$ with the corresponding $C^0$ metrics, i.e., for the restrictions of $a,a'\in\cA$ and $f,f' \in \cF$ to $K$ and $\tilde{K}$ respectively, we let
\begin{align*}
	\| a - a' \|_{C^0} :=  \max_{x \in K} |a(x)-a'(x)|
\end{align*}
and
\begin{align*}
	\| f - f' \|_{C^0} :=  \max_{x \in \tilde{K}} |f(x)-f'(x)|.
\end{align*}
The set $K$ will be equipped with the usual Euclidean $(\ell^2)$ metric. For any metric space $(\Theta,d)$, we will use the notation $\cN(\Theta,d,\delta)$ for its covering numbers.

Before stating our main result on the Rademacher complexity of $\cG$, we introduce some additional notation:
\begin{itemize}
    \item $D := \displaystyle\max_{x,x' \in \tilde{K}}|x-x'|$ is the Euclidean diameter of $\tilde{K}$, also giving the constant $B$ in Theorem~\ref{thm:basic_bound};
    \item for $j = 0,1,2,\ldots$, the functions $\bar{\beta}^j : [0,\infty) \to [0,\infty)$ are defined by
        \begin{subequations}\label{eq:bar_beta}
        \begin{align}
            \bar{\beta}^0(r) &:= r, \\
            \bar{\beta}^1(r) &:= \max_{t \in [T_0,T_1]}\beta(r,t), \\
            \bar{\beta}^{j+1}(r) &:= \bar{\beta}(\bar{\beta}^j(r)), \qquad j = 1, 2, \ldots;
        \end{align}
        \end{subequations}
        since $r \mapsto \beta(r,t)$ is continuous and monotone increasing in $r$ for each $t$, the functions $r \mapsto \bar{\beta}^j(r)$ are lower semicontinuous and monotone increasing;
    \item given the comparison function $\beta$, we define
        \begin{align}\label{eq:Bt}
            \bar{B} := \max_{t \in [T_0,T_1]}\left|\int^t_0 \frac{\partial_+\beta}{\partial r}(0,t-s)\dif s\right|,
        \end{align}
        where $\frac{\p_+ \beta}{\p r}(\cdot,t)$ denotes the right partial derivative of $\beta$ with respect to $r$;
    \item for an arbitrary $\delta \ge 0$, let $\rho_1(\delta),\rho_2(\delta),\rho_3(\delta)$ be the largest nonnegative solutions $\delta_1,\delta_2,\delta_3$ to
        \begin{subequations}\label{eq:rho_delta}
        \begin{align}
            \delta &\ge  \sum_{j=0}^{m-1} \bar{\beta}^j(L_0 \bar{B}\delta_1) \\
            \delta &\ge  \sum_{j=0}^{m-1} \bar{\beta}^j(\bar{B}\delta_2) \\
            \delta &\ge  \bar{\beta}^m(\delta_3)
        \end{align}
        \end{subequations}
        respectively (these solutions exist due to monotonicity and lower semicontinuity of $r \mapsto \bar{\beta}^j(r)$).
\end{itemize}

\begin{theorem}\label{thm:complexity_bound}
    The Rademacher complexity (conditioned on the data $S$) of the class of reconstruction maps $\cG$ satisfies the following bound:
	\begin{align}
		&\cR_n(\cG) \lesssim \frac{1}{\sqrt{n}}\inf \Bigg\{\int^D_0 \Big( m\log \cN(\cA\big|_K,\|\cdot\|_{C^0},\rho_1(\gamma_1\delta)) \nonumber\\
		& \qquad + m\log \cN(\cF\big|_{\tilde{K}},\|\cdot\|_{C^0},\rho_2(\gamma_2\delta)) \nonumber\\
		& \qquad + \log \cN(K,|\cdot|,\rho_3(\gamma_3\delta)) \Big)^{1/2}\dif\delta \Bigg\}, \label{eq:Rad_bound}
	\end{align}
	where  the infimum is over all $\gamma_1, \gamma_2, \gamma_3 > 0$ satisfying $\gamma_1 + \gamma_2 + \gamma_3 = 1$.
\end{theorem}
	
Theorem~\ref{thm:complexity_bound} gives a general recipe for upper-bounding the Rademacher complexity of $\cG$, and we can instantiate the bounds in some specific cases. To that end, we first assume that both $\cA$ and $\cF$ admit finite-dimensional parametrizations, i.e., there exist positive integers $p,q$ and positive real constants $C_{\cA,K}$ and $C_{\cF,\tilde{K}}$, such that
\begin{align*}
	\cN(\cA|_K,\|\cdot\|_{C^0},\delta) &\lesssim \left(\frac{C_{\cA,K}}{\delta}\right)^p, \\ \cN(\cF|_{\tilde{K}},\|\cdot\|_{C^0},\delta) &\lesssim \left(\frac{C_{\cF,\tilde{K}}}{\delta}\right)^q.
\end{align*}
This would be the case if, for instance, the elements of $\cF$ are of the form $x \mapsto f(x;\theta)$, where the vector of parameters $\theta$ takes values in a bounded subset $\Theta$ of $\Reals^q$, and the parametrization is Lipschitz: for any two $\theta,\theta' \in \Theta$,
\begin{align*}
	\|f(\cdot;\theta)-f(\cdot;\theta')\|_{C^0(\tilde{K})} \le L_{\Theta,\tilde{K}} |\theta - \theta'|
\end{align*}
for some constant $L_{\Theta,\tilde{K}} > 0$. Using these facts together with the fact that $\cN(K,|\cdot|,\delta) \lesssim (C_K/\delta)^d$ for some $C_K > 0$ and choosing (say) $\gamma_1 = \gamma_2 = \gamma_3 = \frac{1}{3}$ in Theorem~\ref{thm:complexity_bound}, we get the bound
\begin{align}
		&\cR_n(\cG) \lesssim  \frac{1}{\sqrt{n}}\int^D_0 \Bigg( mp\log \left(\frac{C_{\cA,K}}{\rho_1(\delta/3)}\right) \nonumber\\
		& \quad + mq\log \left(\frac{C_{\cF,\tilde{K}}}{\rho_2(\delta/3)}\right)  + d\log \left(\frac{C_K}{\rho_3(\delta/3)}\right) \Bigg)^{1/2}\dif\delta, \label{eq:findim_bound}
\end{align}
which is the best one can do without further assumptions on $\cF$. We now consider some specific examples.

\begin{example}{\em
    Suppose we only allow positive times --- that is, $\cA \subset \{ a : \Reals^d \to [0, T] \}$ --- and let $\cF$ consist only of uniformly exponentially stable vector fields with magnitude bounded by $L_0$. Then we have, for some $\lambda > 0$,
    \begin{gather*}
        \beta(r,t) = r e^{-\lambda t},\quad\quad \bar{\beta}^j(r) = r, \,\, j = 0,1,2,\ldots \\
        \begin{split}
            \bar{B} &= \max_{t \in [0,T]}\Big| \int_0^t \frac{\p_+ \beta}{\p r}(0,t-s) \dif s\Big| \\
            &= \max_{t \in [0,T]}\Big|\int_0^t e^{-\lambda (t-s)} \dif s\Big| = \frac{1}{\lambda}
        \end{split}
    \end{gather*}
    A simple computation then gives
    \begin{align*}
	   \rho_1(\delta) = \frac{\lambda}{mL_0}\delta, \quad \rho_2(\delta) = \frac{\lambda}{m}\delta, \quad \rho_3(\delta) = \delta,
    \end{align*}
    and therefore we obtain the following from \eqref{eq:findim_bound}:
    \begin{align*}
	   \cR_n(\cG) \lesssim \frac{1}{\sqrt{n}} \left(\frac{m^{3/2}}{\lambda}(C_{\cA,K} L_0 \sqrt{p} + C_{\cF,\tilde{K}} \sqrt{q}) + C_K \sqrt{d}\right).
    \end{align*}}
\end{example}
	
\begin{example}{\em
    Consider all affine vector fields of the form $f(x) = Ax + u$, where the matrices $A \in \Reals^{d \times d}$ and the vectors $u \in \Reals^d$ are uniformly bounded: $\|A\| \le 1$ ($\|\cdot\|$ denoting the spectral norm) and $|u| \le 1$. Since
	$$
	e^{tf\xi} = e^{tA}\xi + \int^t_0 e^{(t-s)A}u \dif s,
	$$
	we can take $\tilde{K} = B^d_2((R+1)e^{m\bar{T}})$, where $R := \displaystyle \max_{\xi \in K}|\xi|$, $\bar{T} := \max\{|T_0|,|T_1|\}$, and $B^d_2(r)$ denotes the $d$-dimensional Euclidean ball of radius $r$ centered at the origin. By Remark~\ref{rem:bump}, we can construct a class $\cF$ of vector fields that are affine on $\tilde{K}$ and vanish outside a compact inflation of $\tilde{K}$, and thus take $L_0 \lesssim Re^{m\bar{T}}$, $L = 1$, $\beta(r,t) = re^{|t|}$, and $\bar{\beta}^j(r) = re^{j\bar{T}}$. 

    The class $\cF$ is parametrized by $\theta = (A,u)$, which takes values in a compact subset of $\Reals^{d \times d} \times \Reals^d$, so $q = d^2 + d$ and $C_{\cF,\tilde{K}} \lesssim Re^{m\bar{T}}$. This shows that the Rademacher complexity $\cR_n(\cG)$ will have an {\em exponential} dependence on the number of layers $m$ and on $\bar{T}$, although this can be removed under additional assumptions, e.g., $0 = T_0 < T_1 = \bar{T}$ and all $A $ being uniformly Hurwitz, i.e., the real parts of all eigenvalues of $A$ are all smaller than $-\lambda$ for some $\lambda > 0$. The latter is a special case of the preceding example.}
\end{example}
	
\begin{example}{\em
    Suppose that $\cF$ consists of all vector fields of the form $f(x) = \sigma(A x + u)$, where $\sigma : \Reals^d \to \Reals^d$ is a fixed bounded Lipschitz nonlinearity, i.e.,
	\begin{align*}
		|\sigma(x)| \le 1 \text{ and } |\sigma(x)-\sigma(y)| \le |x-y|, \qquad \forall x,y \in \Reals^d
	\end{align*}
    and $A \in \Reals^{d \times d}$ and $u \in \Reals^d$ satisfy the same conditions as in the preceding example. (The ODE $\dot{x}=f(x)$ with $f$ of this form is an instance of a continuous-time recurrent neural net \cite{sontag_1997}.) For any such $f$ and any $t$ we have
    \begin{align*}
	   e^{tf}\xi = \xi + \int^t_0 \sigma(Ae^{sf}\xi +u)\dif s,
    \end{align*}
    so Assumption~\ref{as:cF} evidently holds with  $\tilde{K} = B^d_2(R + m\bar{T})$ and $L_0 = L = 1$. As before, the class $\cF$ is parametrized by $\theta = (A,u) \in \Reals^{d \times d} \times \Reals^d$, so $q = d^2 + d$, but now $C_{\cF,\tilde{K}} \lesssim R + m\bar{T}$.

    In contrast to the preceding example, we now have $\beta(r,t) \le \min\{r e^{|t|}, r + 2|t|\}$, so in this setting it is possible for the Rademacher complexity to have {\em polynomial} dependence on $m$ and $\bar{T}$ even without requiring exponential stability.}
\end{example}

\section{Proof of Theorem~\ref{thm:complexity_bound}}\label{sec:complexity}

We first recall a standard technique for bounding expected suprema of random processes indexed by the elements of a metric space, the so-called \textit{Dudley entropy integral} \cite{van_handel_2016}:
	
\begin{lemma}\label{lem:dudley_entropy}
    Let $(Z_\theta)_{\theta \in \Theta}$ be a zero-mean subgaussian random process indexed by a metric space $(\Theta,d)$ --- that is, for all $\theta,\theta' \in \Theta$, $\Ex [Z_\theta] = 0$ and
    \begin{equation*}
        \Pr(|Z_{\theta} - Z_{\theta'}| \geq t) \leq 2\exp\Big(-\frac{t^2}{2 d^2(\theta,\theta')} \Big), \quad \forall t > 0.
    \end{equation*}
    Then we have
    \begin{equation*}
        \Ex \Big[ \sup_{\theta \in \Theta} Z_{\theta} \Big] \lesssim \int_0^D \sqrt{\log \cN(\Theta,d,\delta)} \dif \delta
    \end{equation*}
    where $\cN(\Theta,d,\delta)$ are the $\delta$-covering numbers and 
    \begin{equation*}
        D := \sup_{\theta,\theta' \in \Theta} d(\theta,\theta') 
    \end{equation*}
    is the diameter of $(\Theta,d)$.
\end{lemma}
\noindent To apply the lemma, we consider the class $\cG$ of all reconstruction maps $G_{\ba,\bf,\xi} : \Reals^d \to \Reals^d$ indexed by $\ba \in \cA^m$, $\bf \in \cF^m$, and $\xi \in K$. Fix an $n$-tuple of points $(x_1,\ldots,x_n)$ in $K$ and consider the Rademacher process
\begin{align}\label{eq:ZG}
	Z_G := \frac{1}{\sqrt{n}}\sum^n_{i=1} \eps_i |x_i - G(x_i)|
\end{align}
indexed by $\cG$, where $\eps_1,\ldots,\eps_n$ are i.i.d.\ Rademacher random variables. Then we have the following:

\begin{lemma} The process \eqref{eq:ZG} is subgaussian w.r.t.\ the $C^0(K)$ metric
	$$
	d(G,G') = \|G-G'\|_{C^0(K)} := \max_{x \in K}|G(x)-G'(x)|.
	$$
\end{lemma}
\begin{proof} 
    We need to show that
    \begin{equation*}
        \Pr(|Z_{G} - Z_{G'}| \geq t) \leq 2\exp\Big(-\frac{t^2}{2d^2(G,G')}\Big), \quad \forall t > 0
    \end{equation*}
    for all $G,G' \in \cG$.
    Recall Hoeffding's inequality
    \begin{equation*}
        \Pr(|S_n - \Ex[S_n]| \geq t) \leq 2\exp\Big(-\frac{2t^2}{\sum_{i=1}^n (b_i - a_i)^2}\Big),
    \end{equation*}
    where $S_n := \sum_{i=1}^n Y_i$ and $Y_1,\ldots,Y_n$ are independent random variables, such that $- \infty < a_i \leq Y_i \leq b_i < \infty$ almost surely for each $i$. Let
    \begin{equation*}
        Y_i = \frac{1}{\sqrt{n}} \eps_i \Big(|x_i-G(x_i)|-|x_i - G'(x_i)|\Big),
    \end{equation*}
    so that $Z_{G} - Z_{G'} = \sum_{i=1}^n Y_i$. We have $\Ex[Y_i] = 0$ because each $\eps_i$ is symmetric about zero. We also have
    \begin{equation*}
        -\frac{1}{\sqrt{n}} \|G - G'\|_{C^0} \leq Y_i \leq \frac{1}{\sqrt{n}} \|G-G'\|_{C^0},
    \end{equation*}
    which follows from $-1 \leq \varepsilon_i \leq 1$ and 
    \begin{align*}
		&\big||x_i-G(x_i)|-|x_i-G'(x_i)|\big| \\
		&\leq|G(x_i)-G'(x_i)| \\
        &\leq \max_{x \in K} |G(x)-G'(x)| \\
        &= \|G-G'\|_{C^0(K)}.
    \end{align*}
    Now Hoeffding's inequality reads
    \begin{equation*}
        \Pr(|Z_{G} - Z_{G'}| \geq t) \leq 2 \exp\Big(-\frac{t^2}{2 \| G - G' \|_{C^0}^2}\Big).
    \end{equation*}
    Therefore $Z_G$ is subgaussian with respect to the metric $d(G,G') =  \| G - G' \|_{C^0}$.
\end{proof}
\noindent Since $G(K) \subseteq \tilde{K}$ for every $G \in \cG$, the $C^0(K)$-diameter of $\cG$ is bounded by $D$, the $\ell^2$ diameter of $\tilde{K}$. Therefore, the empirical Rademacher complexity of $\cG$ conditioned on the data $S$ is bounded by
\begin{align}\label{eq:RadG_Dudley}
	\cR_n(\cG) \lesssim \frac{1}{\sqrt{n}}\int^D_0 \sqrt{\log \cN(\cG,\|\cdot\|_{C^0},\delta)}\dif\delta.
\end{align}
Our next order of business is to obtain upper bounds on the covering numbers $\cN(\cG,\|\cdot\|_{C^0},\delta)$.

\subsection{Covering number bounds}\label{sec:covering}

Let $\widehat{\cA}^m := \{\widehat{\ba} : \Reals^d \to [T_0,T_1]^m \}$ be a finite set of encoders that forms a minimal $\delta_1$-net of the metric space $(\cA^m|_K, \|\cdot\|_{\ell^\infty(C^0)})$. That is to say,
\begin{align*}
    & \sup_{\ba \in \cA^m} \min_{\widehat{\ba} \in \widehat{\cA}^m} \|\ba - \widehat{\ba}\|_{\ell^\infty(C^0)}   \\
    & \quad = \sup_{\ba \in \cA^m} \min_{\widehat{\ba} \in \widehat{\cA}^m} \max_{j=1,\dots,m} \max_{x \in K} |a_j(x)-\widehat{a}_j(x)| \le \delta_1.
\end{align*}
Such a finite set exists because $\cA^m|_K$ consists of uniformly bounded and uniformly equicontinuous vector-valued maps supported on a compact set, hence is itself compact.

Similarly, let $\widehat{\cF}^m := \{\widehat{\bf} : \Reals^d \to (\Reals^d)^m\}$ be a finite set of $m$-tuples of vector fields that forms a minimal $\delta_2$-net of $(\cF^m|_{\tilde{K}}, \|\cdot\|_{\ell^\infty(C^0)})$, and let $\cN(\cF^m|_{\tilde{K}},\|\cdot\|_{\ell^\infty(C^0)},\delta_2)$ denote the corresponding covering number. Finally, let $\widehat{K} = \{\widehat{\xi} \in K\}$ be a finite set of points that forms a minimal $\delta_3$-net of $(K, |\cdot|)$ with covering number $\cN(\widehat{K},|\cdot|,\delta_3)$.

\begin{proposition}\label{prop:covering_multiple}
    Let $\widehat{\cG} := \{G_{\widehat{\ba},\widehat{\bf},\widehat{\xi}} : \Reals^d \to \Reals^d : (\widehat{\ba},\widehat{\bf},\widehat{\xi}) \in \widehat{\cA}^m \times \widehat{\cF}^m \times \widehat{K} \} \subset \cG$ where $\widehat{\cA}^m$, $\widehat{\cF}^m$, $\widehat{K}$ are minimal $\delta_1$, $\delta_2$, $\delta_3$-nets of $\cA^m|_K$, $\cF^m|_{\tilde{K}}$, $K$, respectively. Then $\widehat{\cG}$ forms a $\delta$-net of $(\cG, \|\cdot\|_{C^0})$ with
    \begin{align*}
		\delta \le \sum^{m-1}_{j=0} \bar{\beta}^j(L_0\bar{B}\delta_1) + \sum^{m-1}_{j=0} \bar{\beta}^j(\bar{B}\delta_2) + \bar{\beta}^m(\delta_3),
	\end{align*}
	where the functions $\bar{\beta}^j$ are defined in \eqref{eq:bar_beta} and the constant $\bar{B}$ is defined in \eqref{eq:Bt}.
\end{proposition}
\begin{proof}
    Fix any $\ba, \bf, \xi$ and consider the reconstruction map
    \begin{equation*}
        G_{\ba,\bf,\xi} : \Reals^d \to \Reals^d,\quad x \mapsto e^{a_m(x) f_m} \circ \cdots \circ e^{a_1(x) f_1} \xi.
    \end{equation*}
    Let $\widehat{\ba}, \widehat{\bf}, \widehat{\xi}$ be the closest elements in their respective coverings. Then by the triangle inequality
    \begin{align*}
        & \| G_{\widehat{\ba},\widehat{\bf},\widehat{\xi}} - G_{\ba,\bf,\xi} \|_{C^0} \\
        \begin{split}
            &\leq \| G_{\widehat{\ba},\widehat{\bf},\widehat{\xi}} - G_{\ba,\widehat{\bf},\widehat{\xi}}\|_{C^0} \\
            &\quad\quad + \| G_{\ba,\widehat{\bf},\widehat{\xi}}(x) -  G_{\ba,\bf,\widehat{\xi}}(x) \|_{C^0} \\
            &\quad\quad + \| G_{\ba,\bf,\widehat{\xi}} - G_{\ba,\bf,\xi} \|_{C^0}
        \end{split} \\
        &=: D_1 + D_2 + D_3.
    \end{align*}
    We now estimate the three error terms $D_1$, $D_2$, and $D_3$ individually using Lemmas~\ref{lm:1layer_diff}--\ref{lm:mlayer_diff} in the next section.
        
    To estimate $D_1$, given a fixed initial condition $\xi$, fixed vector fields $(f_1,\dots,f_m)$, and two tuples of times $(t_1,\dots,t_m)$ and $(t_1',\dots,t_m')$ such that $|t_j - t_j'| \leq \delta_1$ for every $j = 1,\dots m$, we want to bound the difference between decoder outputs $e^{t_m f_m} \circ \cdots \circ e^{t_1 f_1} \xi$ and $e^{t_m' f_m} \circ \cdots \circ e^{t_1' f_1} \xi$. We can equivalently consider a single tuple of times $(1,\dots,1)$ and two tuples of scaled vector fields $(t_1 f_1,\dots,t_m f_m)$ and $(t_1' f_1,\dots,t_m' f_m)$, which yields the same decoder outputs. In this case, the difference between the vector fields is bounded by $\|t_j f_j - t_j' f_j\|_{C^0(\tilde{K})} = |t_j - t_j'| \|f_j\|_{C^0(\tilde{K})} \leq L_0\delta_1$. Now, if a vector field $f$ admits the comparison function $\beta(r,t)$, then, for any $\tau \in [T_0,T_1]$, the rescaled vector field $\tau f$ admits the comparison function $(r,t) \mapsto \beta(r,\tau t)$. We can therefore apply Lemma~\ref{lm:mlayer_diff} with $(1,\dots,1) \leftarrow (t_1, \dots, t_m)$, $t_j f_j \leftarrow f_j$, $t'_j f_j \leftarrow f'_j$, and comparison function $(r,t) \mapsto \max_{\tau \in [T_0,T_1]}\beta(r,\tau t)$ to get
    \begin{align*}
        D_1 &= \max_{x \in K} |G_{\widehat{\ba},\widehat{\bf},\widehat{\xi}} - G_{\ba,\widehat{\bf},\widehat{\xi}}| \\
        \begin{split}
            &= \max_{x \in K} |e^{\widehat{a}_m(x) \widehat{f}_m} \circ \cdots \circ e^{\widehat{a}_1(x) \widehat{f}_1} \widehat{\xi} \\
            &\quad\quad\quad - e^{a_m(x) \widehat{f}_m} \circ \cdots \circ e^{a_1(x) \widehat{f}_1} \widehat{\xi}|
        \end{split} \\
        & \le \sum^{m-1}_{j=0} \bar{\beta}^j(L_0 \bar{B}\delta_1 ).
    \end{align*}
	
    For $D_2$, we apply Lemma~\ref{lm:mlayer_diff} and use the fact that $\|f_j - \widehat{f}_j \|_{C^0(\tilde{K})} \le \delta_2$ for all $j$, as well as the monotonicity of $r \mapsto \bar{\beta}(r)$, to get
    \begin{align*}
        D_2 &= \max_{x \in K} | G_{\ba,\widehat{\bf},\widehat{\xi}}(x) -  G_{\ba,\bf,\widehat{\xi}}(x) | \le \sum^{m-1}_{j=0} \bar{\beta}^j(\bar{B}\delta_2 ).
    \end{align*}
    Finally, for $D_3$, we apply Lemma~\ref{lm:mlayer_ic} with $k = m$ and use the fact that $|\xi - \widehat{\xi}| \le \delta_3$ to get
    \begin{equation*}
        D_3 = \max_{x \in K} | G_{\ba,\bf,\widehat{\xi}}(x) - G_{\ba,\bf,\xi}(x) | \leq \bar{\beta}^m(\delta_3).
    \end{equation*}
    The proof is completed by taking the supremum over all $(\ba, \bf, \xi) \in \cA^m \times \cF^m \times K$.
\end{proof}
Using this proposition, we can now estimate the covering numbers of $\cG$ as follows: Fix any $\delta \ge 0$ and any $\gamma_1,\gamma_2,\gamma_3 > 0$ such that $\gamma_1 + \gamma_2 + \gamma_3 = 1$. Then, with $\rho_1(\cdot)$, $\rho_2(\cdot)$, $\rho_3(\cdot)$ defined in \eqref{eq:rho_delta}, we have
\begin{align*}
	&\sum^{m-1}_{j=0} \bar{\beta}^j(L_0\bar{B}\rho_1(\gamma_1\delta)) + \sum^{m-1}_{j=0} \bar{\beta}^j(\bar{B}\rho_2(\gamma_2\delta)) \\
	& \qquad \qquad + \bar{\beta}^m(\rho_3(\gamma_3\delta)) \le \delta.
\end{align*}
Therefore, letting $\widehat{\cA}^m$, $\widehat{\cF}^m$, and $\widehat{K}$ be the minimal $\rho_1(\gamma_1\delta)$-, $\rho_2(\gamma_2\delta)$-, and $\rho_3(\gamma_3\delta)$-nets of $\cA^m|_K$, $\cF^m|_{\tilde{K}}$, and $K$ respectively, we see that
\begin{equation*}
\begin{split}
    & \log \cN(\cG,\|\cdot\|_{C^0(K)},\delta) \\
    &\leq \log \cN(\cA^m|_K,\|\cdot\|_{\ell^\infty(C^0)},\rho_1(\gamma_1\delta)) \\
    &\quad + \log \cN(\cF^m|_{\tilde{K}},\|\cdot\|_{\ell^\infty(C^0)},\rho_2(\gamma_2\delta)) \\
    & \quad + \log \cN(K,|\cdot|,\rho_3(\gamma_3\delta)) 
\end{split}
\end{equation*}
We can further upper-bound the quantities on the right-hand side using the fact that
\begin{align*}
	\log \cN(\cA^m|_{K},\|\cdot\|_{\ell^\infty(C^0)},\delta) &\le m \log \cN(\cA|_K,\|\cdot\|_{C^0},\delta)
\end{align*}
and
\begin{align*}
	\log \cN(\cF^m|_{\tilde{K}},\|\cdot\|_{\ell^\infty(C^0)},\delta) \le m \log \cN(\cF|_{\tilde{K}},\|\cdot\|_{C^0},\delta).
\end{align*}
The bound \eqref{eq:Rad_bound} follows by substituting the above covering number estimates into the Dudley entropy integral in \eqref{eq:RadG_Dudley} and then optimizing over all choices of $\gamma_1,\gamma_2,\gamma_3$.

\subsection{Lemmas on iterated flows}\label{ssec:flow_lemmas}

\begin{lemma}\label{lm:1layer_diff}
    For any $f,f' \in \cF$, any $t \in [T_0,T_1]$, and any $\xi \in \Reals^d$ such that $e^{sf}\xi \in \tilde{K}$ for all $s \in [T_0,T_1]$, we have
	\begin{align}
		|e^{tf}\xi - e^{tf'}\xi| \le \bar{B}\| f - f' \|_{C^0(\tilde{K})},
	\end{align}
	where $\bar{B}$ is defined in Eq.~\eqref{eq:Bt}.
\end{lemma}
\begin{proof}
    Assume first $t \ge 0$ and consider the function $H(s) := e^{(t-s)f'} \circ e^{sf} \xi$, which represents flowing along the vector field $f$ for time $s$, then flowing along the vector field $f'$ for the remaining time $t-s$, so that the total time along both vector fields is equal to $t$ for any $s \in [0,t]$. Then when $s = 0$, we have $H(0) = e^{t f'} \xi$, and when $s = t$, we have $H(t) = e^{t f} \xi$. Applying the fundamental theorem of calculus to $H$ and taking the norm yields
    \begin{align*}
        & |e^{t f'} \xi - e^{t f} \xi| \\
        &= |H(0) - H(t)| \\
        &= \Big| -\int_0^t \frac{\dif}{\dif s}H(s) \dif s \Big| \\
        &= \Big| -\int_0^t \frac{\dif}{\dif s} \big(e^{(t-s) f'} \circ e^{s f} \xi\big) \dif s \Big| \\
        &\leq \int_0^t \Big| \frac{\dif}{\dif s} \big(e^{(t-s) f'} \circ e^{s f} \xi\big) \Big| \dif s \\
        &= \int_0^t \Big| \lim_{h \downarrow 0} \frac{1}{h} \big( e^{(t-s-h) f'} \circ e^{(s+h) f} \xi - e^{(t-s) f'} \circ e^{s f} \xi \big) \Big| \dif s \\
        &\leq \int_0^t \limsup_{h \downarrow 0} \frac{1}{h} \Big| e^{(t-s-h) f'}  \big( e^{h f} \circ e^{s f} \xi - e^{h f'} \circ e^{s f} \xi \big) \Big| \dif s
    \end{align*}
    where the last step above follows from continuity of the norm. Using the definition of the function $\beta$ and some properties of limits, we have
    \begin{align*}
        &|e^{t f'} \xi - e^{t f} \xi| \\
        &\!\!\leq \int_0^t \limsup_{h \downarrow 0} \frac{1}{h} \beta\big( | e^{h f} \circ e^{s f} \xi - e^{h f'} \circ e^{s f} \xi |, t-s-h \big) \dif s \\
        &\!\!\leq \int_0^t \limsup_{h \downarrow 0} \frac{1}{h} \beta\big( | f(e^{s f} \xi) - f'(e^{s f} \xi) |h + o(h), t-s-h \big) \dif s \\
        &\!\!\leq \int_0^t | f(e^{s f} \xi) - f'(e^{s f} \xi) | \frac{\p_+ \beta}{\p r} (0, t-s) \dif s,
    \end{align*}
	where the last inequality follows from the fact that $\beta(0,t) \equiv 0$ for all $t$, which further implies that, for any $t$ and for $h \downarrow 0$,
	$$
	\beta(h, t - h) = \frac{\p_+ \beta}{\p r}(0,t)h + o(h).
	$$
	Since $e^{sf}\xi \in \tilde{K}$ by hypothesis, the norm $| f(e^{s f} \xi) - f'(e^{s f} \xi) |$ is bounded by $\| f - f' \|_{C^0(\tilde{K})}$.
    If $t < 0$, we integrate from $t$ to $0$ instead and then take the absolute value.
\end{proof}

\begin{lemma}\label{lm:mlayer_ic}
    For all $k$-tuples $\bt \in [T_0,T_1]^k$ and $\bf \in \cF^k$ and for all points $\xi,\xi'$,
	\begin{align}
		|e^{\bt \bf}\xi - e^{\bt \bf}\xi'| \le \bar{\beta}^k(|\xi-\xi'|).
	\end{align}
\end{lemma}
\begin{proof}
    Here we iteratively apply the definition of the function $\bar{\beta}$ given after Assumption \ref{as:cF}:
    \begin{align*}
        &|e^{\bt \bf}\xi - e^{\bt \bf}\xi'| \\
        &=|e^{t_kf_k} \circ e^{t_{k-1}f_{k-1}} \circ \ldots \circ e^{t_1f_1}\xi \\
        & \qquad \qquad - e^{t_kf_k} \circ e^{t_{k-1}f_{k-1}}  \circ \ldots \circ e^{t_1f_1}\xi'| \\
        &\le\bar{\beta} \big( |e^{t_{k-1}f_{k-1}} \circ \ldots \circ e^{t_1f_1}\xi - e^{t_{k-1}f_{k-1}}  \circ \ldots \circ e^{t_1f_1}\xi'|\big) \\
        &\le \ldots \\
        &\le\bar{\beta}^k (|\xi-\xi'|),
    \end{align*}
    and use the monotonicity of $\bar{\beta}^j(\cdot)$.
\end{proof}

\begin{lemma}\label{lm:mlayer_diff}
    For all $m$-tuples $\bt \in [T_0,T_1]^m$ and $\bf,\bf' \in \cF^m$ and all points $\xi \in K$,
	\begin{align}
		|e^{\bt \bf}\xi - e^{\bt \bf'}\xi| \le \sum^{m}_{j=1} \bar{\beta}^{m-j}(\bar{B} \| f_j-f'_j\|_{C^0(\tilde{K})}). \label{eq:mlayer_diff}
	\end{align}
\end{lemma}
\begin{proof}
    To bound the difference of iterated flows due to different tuples of vector fields $\bf$ and $\bf'$, we build a telescoping sum and apply the triangle inequality:
	\begin{align*}
		&|e^{\bt \bf}\xi - e^{\bt \bf'}\xi| \\
		&=|e^{t_m f_m} \circ e^{t_{m-1}f_{m-1}} \circ \ldots \circ e^{t_1f_1}\xi \\
		& \qquad \qquad - e^{t_m f'_m} \circ e^{t_{m-1}f'_{m-1}} \circ \ldots \circ e^{t_1f'_1}\xi | \\
		&\le |e^{t_m f_m} \circ e^{t_{m-1}f_{m-1}} \circ \ldots \circ e^{t_1f_1}\xi \\
		& \qquad \qquad - e^{t_m f'_m} \circ e^{t_{m-1}f_{m-1}} \circ \ldots \circ e^{t_1f_1}\xi| \\
		& \qquad + \cdots \\
		& \qquad + |e^{t_m f'_m} \circ e^{t_{m-1}f'_{m-1}} \circ \ldots \circ e^{t_2f'_2} \circ e^{t_1 f_1}\xi \\
		& \qquad \qquad - e^{t_m f'_m} \circ e^{t_{m-1}f'_{m-1}} \circ \ldots \circ e^{t_2f'_2} \circ e^{t_1 f'_1}\xi|.
	\end{align*}
    For $j=-1,0,\dots,m-2$, define
    \begin{align*}
    	\xi_j &:= e^{t_{m-j-2}f_{m-j-2}} \circ \dots \circ e^{t_2 f_2} \circ e^{t_1 f_1}\xi.
    \end{align*}
    Then, using Lemma~\ref{lm:mlayer_ic}, we can estimate
    \begin{align*}
    	&\Big| e^{t_m f'_m} \circ \dots \circ e^{t_{m-j} f'_{m-j}}  \circ e^{t_{m-j-1}f'_{m-j-1}} \xi_j \\
    		& \qquad\qquad - e^{t_m f'_m} \circ \dots \circ e^{t_{m-j} f'_{m-j}}  \circ e^{t_{m-j-1}f_{m-j-1}} \xi_j \Big| \\
    	&\quad \leq \bar{\beta}^{j+1} \Big( \big|  e^{t_{m-j-1}f'_{m-j-1}} \xi_j -   e^{t_{m-j-1}f_{m-j-1}} \xi_j \big|\Big).
    \end{align*}
    Moreover, since $e^{tf}\xi_j \in \tilde{K}$ for all $f \in \cF$ and all $t \in [T_0,T_1]$ by virtue of Assumption~\ref{as:cF}, we can apply Lemma~\ref{lm:1layer_diff} to obtain
    \begin{align*}
    	&\bar{\beta}^{j+1} \Big( \big|  e^{t_{m-j-1}f'_{m-j-1}} \xi_j -   e^{t_{m-j-1}f_{m-j-1}} \xi_j \big|\Big) \\
    	& \qquad \qquad \leq \bar{\beta}^{j+1} \Big( \bar{B} \| f_{m-j-1} - f'_{m-j-1}\|_{C^0(\tilde{K})}\Big).
    \end{align*}
    Summing over $j$ and reindexing, we get \eqref{eq:mlayer_diff}.
\end{proof}
    
\section{Conclusions}\label{sec:conclusions}

In this work we have developed an encoder-decoder architecture for learning immersed submanifolds based on the construction in Sussmann's orbit theorem, and provided high-probability bounds on its generalization error. Capitalizing on the natural recursive structure present in differential equations allows for complex generative models to be built from comparatively simple vector field parameterizations. This architecture generalizes some well-known learning algorithms to the nonlinear setting. For example, principal component analysis (PCA) is recovered by choosing orthogonal projection encoders and taking $\cF$ to consist of constant vector fields. Various nonlinear generalizations of PCA can also be recovered by choosing an appropriate family of vector fields that enables the application of a nonlinear kernel to the input. Thus, applications to parametric manifold learning and nonlinear dimensionality reduction are practicable. We can also find utility in this architecture for high-dimensional sampling problems and generative modeling, since the latent space (of times) can be much simpler and of lower dimension than the complexity of features produced by the decoder.
	
This formulation naturally lends itself to a neural net implementation based on recent work on \textit{neural ODEs} \cite{chen_2018}, since the  compositional structure of deep neural nets --- i.e., composition of nonlinear layers --- is clearly manifested in the recursive structure of ODE flows, $e^{tf} = e^{(t-s)f} \circ e^{sf}$; an even richer class of models can be obtained by allowing \textit{controlled} ODEs. Indeed, the idea of treating flow maps as computational units and composing them to achieve more complex maps can also be found in previous geometric control literature. For instance, Agrachev and Caponigro \cite{agrachev_caponigro_2009} show that, under some conditions, any diffeomorphism isotopic to the identity can be represented as a composition of flows of finitely many smooth vector fields.  
	
A promising control-theoretic application of the proposed architecture is in the context of motion planning \cite{li_canny_1993} using hybrid (neurosymbolic) systems, both in continuous time \cite{Manikonda_99} and in discrete time \cite{Sontag_86}. Indeed,  we can associate to the collection $\{f_1,\ldots,f_m\}$ of vector fields learned by our procedure a controlled system $\dot{x} = f(x,u)$ with \textit{finite} control set $U = \{1,\ldots,m\}$ and $f(\cdot,u) \equiv f_u(\cdot)$. With this system, we can then think about approximately steering the initial point $\xi$ to an arbitrary target point $x \in K$ by successively appling the controls in $U$ according to the `program' $(a_1(x),\ldots,a_m(x))$. If the target point is sampled at random according to the same distribution $\mu$ that was used during training, we can then guarantee that the resulting steering plan will be nearly optimal (in a given class) with high probability. Going beyond the motion planning problem, we can consider other control-theoretic applications by treating the learned vector fields $f_1,\ldots,f_m$ as the generators of a control group in the sense of Lobry \cite{Lobry_1973}; various natural questions can then be phrased in terms of the Lie algebra generated by these $m$ vector fields.

A study of the inherent expressiveness of dynamic layers from a geometric control perspective is  also a promising direction for future work. Developing numerical schemes based on geometric integrators and applying order reduction techniques to efficiently implement the models discussed here as finite-depth nets is another topic of interest.
	
\section*{Acknowledgments}

The authors would like to thank Ali Belabbas, Ramon van Handel, Eduardo Sontag, and Matus Telgarsky for stimulating discussions that helped crystallize the ideas that led to this work.

\bibliography{learning_immersions_arxiv.bbl}

\end{document}